\documentclass[11pt]{article}

\usepackage[T1]{fontenc}
\usepackage[utf8]{inputenc}
\usepackage{authblk}

\usepackage{graphicx}
\usepackage{multirow}
\usepackage{amsmath,amssymb,amsfonts}
\usepackage{amsthm}
\usepackage{mathrsfs}
\usepackage[title]{appendix}
\usepackage{xcolor}
\usepackage{textcomp}
\usepackage{manyfoot}
\usepackage{booktabs}
\usepackage{algorithm}
\usepackage{algorithmicx}
\usepackage{algpseudocode}
\usepackage{listings}
\usepackage{comment}
\usepackage{bbm}
\usepackage{url}
\usepackage[
  a4paper,
  left=2cm,
  right=2cm,
  top=2.5cm,
  bottom=2.5cm
]{geometry}

\newcommand{\bR}{\mathbb{R}}

\newcommand{\bN}{\mathbb{N}}

\newcommand{\bE}{\mathbb{E}}

\newcommand{\cB}{\mathcal{B}}

\providecommand{\keywords}[1]{%
  \par\medskip\noindent\textbf{\textit{Keywords}:} #1
}

\theoremstyle{plain}
\newtheorem{theorem}{Theorem}
\newtheorem{proposition}[theorem]{Proposition}
\newtheorem{lemma}[theorem]{Lemma}
\newtheorem{corollary}[theorem]{Corollary}

\theoremstyle{definition}
\newtheorem{definition}{Definition}

\theoremstyle{remark}

\title{Distribution of hitting times for dissipative random dynamical systems
on $\mathbb{R}^d$, with application to stochastic gradient descent}

\author[1]{Stefano Galatolo\thanks{\texttt{stefano.galatolo@unipi.it}}}
\author[2]{Stéphane Chrétien\thanks{\texttt{stephane.chretien@univ-lyon2.fr}}}

\affil[1]{Dipartimento di Matematica, Università di Pisa,
Largo Pontecorvo 5, 56127 Pisa, Italy}

\affil[2]{Laboratoire ERIC, Université Lumière Lyon 2,
5 av. Pierre Mendès-France, 69300 Bron, France}

\date{}

\begin{document}

\maketitle


\abstract{Machine Learning and more specifically Deep Learning involves solving large scale nonconvex optimization problems. Several algorithms have been proposed in the literature, that seem to achieve satisfactory practical efficiency for difficult instances, the Stochastic Gradient Method  being the most rudimentary, while still outperforming more recent algorithms at a number of learning tasks.  

A major open question about the current methods used in deep learning is to understand their convergence properties. Following a line of previous works about the long time behavior of gradient-type algorithms,
we present a new approach for studying the asymptotic properties of a wide family of methods from an ergodic theoretical viewpoint. Our main results include a study of the expected time for a stochastic optimisation algorithm to reach a certain small neighborhood of a minimizer and show that this reaching time distributes exponentially around its average, which is given by the inverse of the stationary measure of the target. The assumptions on the Stochastic Gradient noise include the Gaussian and the Sub-Exponential assumptions.}




\maketitle

\section{Introduction}

Deep Learning is a topic of great impact on science and technology, which drives the increase of potential applications. 
The study of optimization algorithms has been a topic of vastly renewed interest, driven by the needs of efficient loss minimising techniques in the fields of Data Science, Statistics and Machine Learning \cite{wright2022optimization}. In Machine Learning, and especially Deep Learning \cite{lecun2015deep,zhang2023dive,prince2023understanding}, mappings denoted by $\phi_{\theta}$: $\mathbb{R}^{d_x}\mapsto \mathbb R^{d_y}$, parametrised by a high-dimensional vector $\theta\in\mathbb R^{d_\theta}$. An appropriate value for $\theta$, here denoted by $\hat \theta$, is to be "learned from the data", which in practice is achieved by finding $\hat \theta$ as a solution of a minimisation problem of the form 
\begin{align}
    \min_{\theta \in \Theta} \ f(\theta) 
    \label{optmain}
\end{align}
where $f$ is called the empirical loss and can be written as a sum of a very large, up to billions, number of individual loss functions $f_i$ 
\begin{align}
    f(\theta) & = \frac1{n} \sum_{i=1}^n \ f_i(\theta)
    \label{loss}
\hspace{1cm} \text{ with } \hspace{1cm}
    f_i(\theta)  = \ell(y_i,\phi_{\theta}(x_i))
\end{align}
where $\ell$: $\mathbb R^{d_y}\times \mathbb R^{d_y}$, usually convex and twice differentiable in the second variable, is used to measure the discrepancy between each $y_i$ and its approximation using $\phi_{\theta}(x_i)$ for all $i=1,\ldots,n$. Using this optimization approach is motivated by attempting to  optimally replicate a hypothetical relationship between inputs and outputs, based on a finite sample set of observed input-output couples $(x_1,y_1),\ldots,(x_n,y_n)$,  $x_i\in \mathbb R^{d_x}$ and $y_i\in \mathbb R^{d_y}$, $i=1,\ldots,n$.

Since the earlier stages of Deep Learning, the mapping $\phi_\theta$ was a often assumed of the compositional form 
\begin{align}
    \phi_\theta & = \phi^{(L)}_{\theta^{(L)}}\circ \phi^{(L-1)}_{\theta^{(L-1)}} \circ \cdots \circ \phi^{(1)}_{\theta^{(1)}}
    \label{compo}
\end{align}
with $\theta =(\theta^{(L)},\theta^{(L-1)},\ldots,\theta^{(1)})$ and 
\begin{align*}
    \phi^{(l)}_{\theta^{(l)}}(\cdot) & = \sigma(W^{(l)} \cdot +b^{(l)})
\end{align*}
where $\theta^{(l)}=(W^{(l)},b^{(l)})$ with $W^{(l)}\in \mathbb R^{d_{l+1}\times d_l}$, $b^{(l)}\in \mathbb R^{d_{l+1}}$, $\sigma: \mathbb R\mapsto \mathbb R$ is increasing and applies pointwise to vectors. The theory of such compositional mappings is well summarized in \cite{poggio2024compositional}. Recent architectures based on including "skip connections" such as in residual neural networks, or attention layers in transformer architectures also share this compositional structure. In Residual Networks, each block is a mapping $G_k = Id + F_k$, and the network is a composition of these blocks.
Then an $L$-block ResNet is exactly the successive composition
\[
x_L = G_{L-1}\circ G_{L-2}\circ \cdots \circ G_0(x_0)
     = (\mathrm{Id}+F_{L-1})\circ \cdots \circ (\mathrm{Id}+F_0)(x_0).
\]
Many more architectures have been devised for various tasks that are key to the development of AI, such as the ones based on the attention mechanism.

The surprising performance of Deep Machine Learning models for a large number of tasks is one of the most studied mysteries in contemporary science and engineering,  which triggered an unprecedented amount of theoretical and practical research in the fields of theoretical computer science and applied mathematics. Several works have recently appeared to give a broad and deep understanding of the mathematical underpinnings of the efficiency of Deep Learning, and  \cite{Kutyniok2022MathematicalAspects}, \cite{Jentzen2024MathematicalIntroDL} and \cite{spiliopoulosmathematical} provide excellent pedagogical introductions. From a high level perspective, the mathematical study of Deep Learning networks, mostly pertains to three different sub-fields : \begin{itemize}
    \item Optimization theory: 
    \begin{itemize}
        \item Why do simple methods such as stochastic gradient methods, ADAM, MUON, etc, reach  parameter regions associated with good performance in prediction, despite the non-convexity of problem \eqref{loss} ?
        \end{itemize}
    \item Approximation theory: 
    \begin{itemize}
        \item What are the approximation properties of compositional functions of type \eqref{compo} ? 
        \item What is the impact of the number of layers and the number of neurons per layer on the prediction performance of a neural network ? 
    \end{itemize}
    \item Generalisation theory: 
    \begin{itemize}
        \item what are the statistical guarantees on the prediction problem with a new dataset drawn from the same distribution as the training set's distribution ? 
        \item Will the network still perform approximately correctly on new unseen data drawn from a slightly different distribution (distribution shift) ?
    \end{itemize}
    
\end{itemize}
While impressive progress has been made in the last two topics in the last decade, namely in Approximation theory and Generalisation theory, leading to a large number of new findings, theoretical results concerning the performance of Optimization methods are drastically more scarce.

\subsection{The problem of optimising the weights in deep learning}

\subsubsection{The deep learning landscape.}
A parallel line of research has focused on characterising the \emph{loss landscape} of deep neural networks, often described as a high-dimensional ``wilderness'' with numerous saddle points, wide valleys, and flat basins of attraction. Recent empirical and theoretical results suggest that, despite their non-convexity, these landscapes may become surprisingly benign when the width of the network is  large compared to the sample size \cite{liu2022loss,belkin2021fit,wilson2025deep}: most critical points are either saddle points or global minima with similar function values, and the minima found by many optimization algorithms tend to be \emph{flat}, a property saying that the hessian is very small in a large basin, and which is strongly correlated with good generalization \cite{hochreiter1997flat,keskar2017large,neyshabur2017exploring,haddouchepac}. Moreover, studies on mode connectivity reveal that independently trained solutions are often connected by low-loss paths, indicating a rich but structured geometry \cite{garipov2018loss,draxler2018essentially}. Recent results about generalization suspect that the optimization methods also implicitly regularize the problem, converging to solutions that are less complex than the initial formulation of the optimization problem \eqref{loss} would suggest \cite{ma2021sobolev,wilson2025deep}. 

For Transformer-based architectures, a similar picture is emerging. Works such as \cite{ahn2023transformerlandscape,fort2023understanding} analyze the loss surface of large language models and highlight scaling laws governing curvature and gradient variance as the number of parameters grows. 

These findings motivate the use of algorithms that exploit this structure, for instance, adaptive optimisers that seem to outperform all previous standard algorithms. On the theoretical side, mathematically grounded justifications are still lacking for their convergence properties and about the time needed to converge to a satisfying target. 

The goal of the present work is to elucidate some of their performance capabilities by approaching their behavior through the lens of dynamical systems theory, a powerful approach to establishing interesting hitting time results for a wide class of optimization methods.

\subsubsection{Optimization algorithms for non-convex minimization and the probability of approaching the global minimum.}\label{sec:proba}

In practice, solving problem~\eqref{optmain} requires efficient algorithms able to scale to datasets of size $n \ge 10^6$ and parameter vectors $\theta$ with up to billions of entries. A vast literature has therefore emerged on optimization algorithms such as Stochastic Gradient Descent (SGD) and its numerous variants (Momentum, ADAM \cite{kingma2014adam}, RMSProp \cite{kurbiel2017training}, Muon \cite{jordan2024muon}, etc.), which are now standard tools in Deep Learning. Despite their empirical success, these algorithms operate in highly non-convex settings, where convergence guarantees are delicate and often asymptotic. Understanding  their convergence properties  from a quantitative point of view is crucial to explain their empirical performance and guide the design of new, more robust optimisers.

One quantitative way to assess the performance of a given optimization algorithm is to give an estimate of the probability that the algorithm reaches a satisfying result in a given time. 

In this direction, one fruitful approach to analyse the behaviour of stochastic optimization algorithms is to view them as random dynamical systems. From this perspective, the trajectory $(\theta_k)_{k \ge 0}$ generated by several optimization algorithms can be seen as a Markov process evolving in parameter space. Studying hitting times of relevant sets,  for instance, neighbourhoods of critical points or low-loss sublevel sets, provides valuable insight into the speed at which the algorithm explores the loss landscape and escapes saddle points. 

Approaching the problem from a dynamical point of view allows estimating the probability that  the algorithm visit a certain region of the phase space in a certain window of time (e.g. a neighborhood of the global minimum).
Coarser or finer estimates for this probability are  available in the literature and can be approached by different methods. 
The works \cite{azizian2024long,azizian2025global}  combine stochastic approximation theory, dynamical systems methods and large deviation principles to obtain information on the SGD’s behaviour in non-convex landscapes. In \cite{azizian2024long}, they characterise the \emph{long-run distribution} of SGD under constant step size, showing that 
the iterates concentrate exponentially around connected components of critical points.
 They characterize the exponential concentration of the long-run distribution of SGD among the critical components, according to effective energy levels determined by both the objective landscape and the noise statistics.

\subsection{Our approach: the hitting time.}

Beyond estimating the probability that an optimization algorithm lies in a satisfactory region after a long run, it is also important to understand the time needed to first reach such a region, as it will be an estimate for how much time and computational effort is necessary to reach a satisfactory target.

In the context of stochastic optimization, a satisfactory target may be a neighborhood of a minimizer, a low-loss sublevel set, or more generally a region of parameter space associated with good performance. 

One of the motivations of the present work is to relate the first reaching time of a satisfactory target to the stationary probability of that target, in connection with the probabilistic estimates discussed in Section \ref{sec:proba}. When the algorithm is modeled as a random dynamical system with stationary measure \(\mu\), the relevant object is the hitting time
\[
    \tau_A := \inf\{n\geq 1 : X_n \in A\},
\]
where \(A\) is the target region. This naturally connects stochastic optimization with the theory of hitting, or waiting, times in dynamical systems.

In this theory, one aims to estimate quantitatively the time needed for a typical trajectory to enter a small target set. 
A guiding principle, often valid under suitable mixing assumptions, is that the typical time scale for hitting a rare target is the inverse of its stationary measure:
\begin{equation}
    \tau_A \ \text{is typically of order}\ \frac{1}{\mu(A)} .
    \label{eq:inverse-measure-principle}
\end{equation}

More precise results are usually obtained in an asymptotic regime where the target sets shrink. Given a sequence of targets \(A_n\) with \(\mu(A_n)\to 0\), one studies the distribution of \(\tau_{A_n}\) after normalization by \(\mu(A_n)\). 
This approach is asymptotic in the size of the target, but it provides a precise description of the typical scale and distribution of the hitting time. This is the point of view adopted in the present paper.

\

In the machine learning literature, two recent works are particularly related to the estimation of the time needed to reach a certain optimization target.  \cite{JMLR:v21:19-327} study Stochastic Gradient Langevin Dynamics and prove finite-time upper bounds for the first time at which the algorithm reaches approximate stationary points. Their analysis is based on direct descent estimates, in a regime where the gradient dominates the noise and gives explicit dependence on the smoothness of the objective, the dimension, the noise strength, and the step-size schedule. In particular, their first-order result concerns the hitting time of a region of the form
\(
    \{x:\|\nabla F(x)\|\le \varepsilon\}.
\)

A different but closely related perspective is developed in \cite{azizian2025global}, where  the global convergence time of SGD in non-convex landscapes via large-deviation theory is studied. Their goal is to estimate the time needed by SGD, starting from a prescribed initialization, to reach a neighborhood of the global minimizers. In the small-step-size regime, they obtain  estimates showing that this time is governed by an energy quantity associated with the transition graph of the loss landscape. Their results therefore describe a metastable regime in which the dominant cost is the rare transition across the most difficult barriers separating the initial condition from the set of global minimizers.

The results of the present paper are in some sense complementary to these contributions. The estimates of \cite{JMLR:v21:19-327} are optimization-oriented hitting-time bounds where the target set is a low gradient region. In their first-order result, the Langevin noise must be sufficiently small relative to the target accuracy; consequently, this result does not describe the rare-event regime in which the stochastic dynamics is fixed and the target set shrinks. The results of \cite{azizian2025global} address a different asymptotic regime, where the step size plays the role of a small-noise parameter and the relevant time scale is determined by large-deviation costs and barrier crossing between basins.

Here, instead, we fix a dissipative random dynamical system with a stationary measure and study the first entrance time into shrinking targets. In this regime we will show that the relevant scale is obtained from the stationary measure of the target itself. Our main result proves that, for targets \(B_n\) satisfying
\(
    n\mu(B_n)\to \tau,
\)
the hitting time satisfies an exponential law and has mean asymptotic to \(1/\mu(B_n)\) (see Theorem \ref{thm:gen2}).
Thus our approach describes the rare-event statistics of small targets for the stationary dynamics, while the works cited above describe, respectively, finite-time descent to approximate stationary regions and metastable global transitions in the small-noise regime.

It is worth to remark that there exist mixing systems (with relatively slow mixing rate) for which the hitting time of small targets is much bigger, in average and also almost everywhere, than the inverse of the measure of the target (see \cite{GP} for deterministic examples, and   \cite{GSR15}, \cite{G_JEP} for examples which can be interpreted as random rotations). Hence the question we address, from the general point of view is not trivial.

The link between first hitting times of shrinking targets and extreme value laws is now classical and has been developed in a general dynamical
framework by \cite{FreitasFreitasTodd2010}. In the random setting,
exponential laws for hitting times have been obtained for some classes of
systems, in particular for random subshifts of finite type and for random
expanding maps, where one typically has strong mixing properties and
dynamically defined target sets such as cylinders. A representative result
in this direction is due to \cite{RousseauSaussolVarandas2015}, establishing
quenched exponential laws for hitting times in classes of random dynamical
systems with super-polynomial decay of correlations, including random
subshifts and random expanding maps. Closely related results on hitting
times and periodicity in random dynamics, again in a symbolic quenched
setting, were obtained in \cite{RT15}. A theory of rare events for
deterministic and random dynamical systems, with particular emphasis on the
role of periodicity, clustering, and extremal indices, is developed in
\cite{AFV15}. Related results on extreme value statistics for randomly
perturbed dynamical systems can also be found in \cite{FFLTV}, while a
complementary line of work is based on perturbative spectral methods for
transfer operators, as developed by Keller, which provide a powerful route
to exponential hitting-time laws and extremal indices \cite{Keller2012}.
In \cite{FreitasFreitasVaienti2017}, the authors develop limit laws for a class of non-stationary stochastic processes and apply them in particular to sequential dynamical systems generated by uniformly
expanding maps, as well as to some classes of random dynamical systems. 
A direction particularly close in spirit to the present work is the recent
paper \cite{OTHERPAPER}, which studies extreme value laws and Poisson
statistics for discrete-time samplings of stochastic differential equations
on \(\mathbb R^n\). That paper was a direct source of inspiration for our
approach: in both cases, the key mechanism is the spectral analysis of
annealed transfer operators acting on suitable spaces of densities. The two
frameworks are, however, different in scope. In \cite{OTHERPAPER}, the
randomness is generated by a stochastic differential equation,
whereas here we work directly with discrete-time random dynamical systems.
 More importantly, our setting is
designed to accommodate a wider class of random perturbations, including bounded
or unbounded noise and state-dependent noise, which can be considered in suitable formalizations of a random dynamical model of  stochastic gradient descent algorithms.

\section{Main Results}
\label{sec:mainresults}

We introduce a class of random dynamical systems on $\mathbb{R}^d$ having a kind of dissipative behavior, in the sense that their dynamics tend to concentrate typical orbit in a limited region of the space. For this class we will prove exponential hitting time distribution in small targets under the assumption of Gaussian and Sub-Exponential noise. 
We will then see in Section \ref{MR2} that this class contains idealized models of the behavior of the Stochastic Gradient Descent  showing the exponential hitting time distribution for these models.

\subsection{Dissipative Random Dynamical Systems on $\mathbb{R}^d$ and their hitting time distribution.}
\label{MR1}

We introduce a class of random dissipative dynamical systems on \(\mathbb R^d\), for which we will establish an
exponential distribution of hitting times.

To model the randomness in the system, let \((\mathcal E,\pi)\) be a
probability space and let
\[
\Omega:=\mathcal E^{\mathbb N},
\qquad
\mathbb P:=\pi^{\mathbb N}
\]
be the associated shift space. Each \(\omega\in\Omega\) is an infinite
sequence \(\omega=(\omega_i)_{i\in\mathbb N}\) of independent random
variables.

We then consider a family of measurable maps
\[
T_\omega:{\mathbb R}^d\to {\mathbb R}^d,
\qquad \omega\in\mathcal E,
\]
and define the random dynamical system by
\begin{equation}
X_{i+1}=T_{\omega_i}(X_i),
\label{RS-new}
\end{equation}
where \(X_i\in{\mathbb R}^d\). The sequence \(\{X_i\}_{i\ge 0}\) defines the
\emph{random orbit} of the system, which depends on the initial condition
\(X_0\in{\mathbb R}^d\) and on the noise realization \(\omega\in\Omega\).

\begin{definition}
\label{def:koopman-new}
Let \(X_t(z)\) denote the random orbit of \eqref{RS-new} at time \(t\) with
initial condition \(z\in{\mathbb R}^d\). The annealed Koopman operator
\[
 P:L^\infty({\mathbb R}^d)\to L^\infty({\mathbb R}^d)
\]
associated with the system \eqref{RS-new} is defined by
\[
(P\varphi)(z):=\mathbb E[\varphi(X_1(z))].
\]
\end{definition}

From the definition it directly follows that
\begin{equation}
\Vert P\phi \Vert _{\infty }\leq \Vert \phi \Vert _{\infty }.  \label{inft}
\end{equation}

Dually, we can define  the associated annealed transfer operator, acting on densities.
\begin{definition}
\label{def:transfer-new}
The annealed transfer operator
\[
 L : L^1({\mathbb R}^d)\to L^1({\mathbb R}^d)
\]
associated with the system \eqref{RS-new} is defined as the operator dual to
the annealed Koopman operator \( P\), namely by the relation
\begin{equation}
\int_{{\mathbb R}^d} (L f)(z)\,\varphi(z)\,dz
=
\int_{{\mathbb R}^d} f(z)\,(P\varphi)(z)\,dz
\label{eq:dual-transfer}
\end{equation}
for every \(f\in L^1({\mathbb R}^d)\) and every
\(\varphi\in L^\infty({\mathbb R}^d)\).
\end{definition}

Throughout the paper, we make the following assumptions.

\paragraph{Standing assumptions.} We assume that the dynamics admits a certain "deterministic map + variable noise" representation, and furthermore the resulting system is dissipative. 
More precisely we suppose:

\begin{description}

\item[(a)] \textbf{The system has smoothly distributed noise.}

 The dynamics can be written in the form
\begin{equation}
X_{i+1}
= T(X_i)+ P(X_i,\omega_i)
= T_{\omega_i}(X_i),
\label{RS-q}
\end{equation}
where \((\omega_i)_{i\ge 0}\) is an i.i.d.\ sequence,
\( T:{\mathbb R}^d\to{\mathbb R}^d\) is measurable, and
\( P:{\mathbb R}^d\times {\mathcal E}\to {\mathbb R}^d\)
describes the effective  noise. We assume that, for each fixed
\(z\in{\mathbb R}^d\), the random variable
\( P(z,\omega)\) is distributed with a density
\[
\hat\rho_z:{\mathbb R}^d\to\mathbb R,
\]
and that these densities are uniformly \(C^1\), namely
\begin{equation}
\sup_{z\in{\mathbb R}^d}\|\hat\rho_z\|_{C^1}<\infty.
\label{uniform-C1-noise}
\end{equation}

\item[(b)] \textbf{Dissipativity / Foster--Lyapunov condition.}
For \(A>0\) and \(\lambda>0\),  define the exponential Lyapunov Function
\[
V(z):=A e^{\lambda\|z\|}.
\]
We suppose $A$ is chosen in a way that $V(z)\geq 1+||z||^2$ and there exist constants \ \(\kappa\in(0,1)\), \(b>0\), and a
compact set \(K\subset{\mathbb R}^d\) such that
\begin{equation}
{ P} V(z)\le \kappa V(z)+b\,\mathbf 1_K(z)
\qquad\text{for all }z\in{\mathbb R}^d.
\label{FL-new}
\end{equation}

\item[(c)] \textbf{Existence and uniqueness of the stationary measure.}
The system \eqref{RS-new} admits a unique stationary probability measure
\(\mu\) on \({\mathbb R}^d\), absolutely continuous with respect to the Lebesgue measure.

\end{description}

\paragraph{Main results: hitting-time statistics and extreme events for a class of random dynamical systems.}
For a random dynamical system satisfying the above assumptions, we establish
the following theorem, which characterizes the asymptotic distribution of
first hitting times for shrinking target sets.

\begin{theorem}
\label{thm:gen2}
Let \(\tau>0\), and let \(X_i(z)\) denote the random orbit of
\eqref{RS-new} at time \(i\) starting from \(z\in{\mathbb R}^d\), with random
seed \(\omega\in\Omega\).

Suppose the stationary measure \(\mu\) of \eqref{RS-new} has a density
\(f_0\). Let $x_0\in{\mathbb R}^d $ be such that \(f_0(x_0)>0\). Let
$
B_n:=B(x_0,r_n),
$
satisfying
\begin{equation}
n\,\mu(B_n)\longrightarrow \tau.
\label{eq:tau-scaling}
\end{equation}
Then
\begin{equation}
\lim_{n\to\infty}
\mathbb P\otimes\mu\Big(
\{(\omega,z): X_i(z)\notin B_n,\ \forall i \text{ such that }0\le i\le n-1\}
\Big)
=
e^{-\tau}.
\label{eq:main-hitting-law}
\end{equation}
\end{theorem}

The statement of Theorem~\ref{thm:gen2} can also be reformulated in the
language of hitting-time distributions as follows.

\begin{corollary}
\label{cor:exp-law}
Assume the hypotheses and notation of Theorem~\ref{thm:gen2}, and let
\((B_n)_{n\ge 1}\) be the corresponding sequence of target sets. For each \(n\), define the hitting time
\[
\tau_{B_n}(\omega,z):=\inf\{k\ge 1: X_k(z)\in B_n\}.
\]
Then, with respect to \(\mathbb P\otimes\mu\), the rescaled hitting times
\(\mu(B_n)\tau_{B_n}\) converge in distribution to an exponential random
variable of parameter \(1\). Equivalently, for every \(t\ge 0\),
\begin{equation}
\lim_{n\to\infty}
(\mathbb P\otimes\mu)\!\left(
\tau_{B_n}>\frac{t}{\mu(B_n)}
\right)
=
e^{-t}.
\label{eq:exp-law}
\end{equation}
\end{corollary}

A standard consequence of the exponential law above is the following.

\begin{corollary}
\label{cor:mean-hitting}
Under the hypotheses of Theorem~\ref{thm:gen2}, the mean hitting time to
\(B_n\) satisfies
\begin{equation}
\lim_{n\to\infty}
\mu(B_n)\,
\mathbb E_{\mathbb P\otimes\mu}\big[\tau_{B_n}\big]
=
1.
\label{eq:mean}
\end{equation}
In particular,
\[
\mathbb E_{\mathbb P\otimes\mu}\big[\tau_{B_n}\big]
\sim
\frac{1}{\mu(B_n)}
\qquad\text{as }n\to\infty.
\]
\end{corollary}

{In Section \ref{MR2} we will apply these abstract results to a model of machine learning optimization algorithms: the Stochastic Gradient Descent.

\paragraph{Strategy of proof of the main result.}  
The proof of Theorem~\ref{thm:gen2} will be detailed in Section \ref{app:repfo}, but we  outline the main ideas here. The proof relies on spectral perturbation techniques following Keller and Liverani~\cite{kellerliverani1999,Keller2012}.  
    The key idea is to represent the evolution of densities under the dynamics via the associated transfer operator ${L}$, identify Banach spaces on which ${L}$ satisfies a Lasota–Yorke inequality, and exploit the existence of a spectral gap.  
This spectral decomposition enables precise control of rare events and the distribution of hitting times.

\subsection{Application to Stochastic Gradient Descent}\label{MR2}

We now discuss how the abstract framework developed above applies to a
simple stochastic gradient descent (SGD) model on \(\mathbb R^d\).
We will see this stochastic optimisation algorithm as a random dynamical system.

Let
\[
f:\mathbb R^d\to\mathbb R
\]
be a differentiable function, to be interpreted as an empirical loss or
objective function. In practical optimisation algorithms, the gradient
\(\nabla f(x)\) is typically not evaluated exactly at each iteration. Instead,
one uses a stochastic approximation, for instance obtained from minibatch
sampling. This naturally leads to a random iteration, in which the
deterministic gradient step is perturbed by a random fluctuation.

From this point of view, stochastic gradient descent can be viewed as a
random dynamical system on the phase space \(\mathbb R^d\), where the
deterministic part is given by the gradient update and the randomness models
the fluctuation of the stochastic gradient around its mean value.

\paragraph{A simple additive-noise model for SGD.}
Fix a step size \(\delta>0\). We consider the deterministic map
\begin{equation}
T(x):=x-\delta \nabla f(x),
\label{eq:SGD-map}
\end{equation}
which corresponds to one step of gradient descent with step size \(\delta\).
We then perturb this map by an additive random term and define the random
iteration
\begin{equation}
X_{n+1}=T(X_n)+\xi_{n+1},
\label{eq:SGD-random}
\end{equation}
where \((\xi_n)_{n\ge 1}\) is an i.i.d.\ sequence of random vectors in
\(\mathbb R^d\).

This is a random dynamical system of the form considered in the present
paper,
\[
X_{n+1}=T(X_n)+P(X_n,\omega_n),
\]
with
\[
P(x,\omega)=\omega.
\]

\paragraph{A dissipativity condition on the gradient.}
Assume that \(f\in C^1(\mathbb R^d)\) and that there exist constants
\(C_0>0\) and \(R_0>0\) such that
\begin{equation}
\langle \nabla f(x),x\rangle \ge C_0\|x\|
\qquad\text{for all }x\in\mathbb R^d\text{ with }\|x\|\ge R_0.
\label{eq:SGD-gradient-diss}
\end{equation}
Then the deterministic gradient map is dissipative. Indeed,
\[
\|T(x)\|
=
\|x-\delta \nabla f(x)\|
\le
\|x\|-\delta\,\frac{\langle \nabla f(x),x\rangle}{\|x\|}
\le
\|x\|-\delta C_0
\]
for every \(\|x\|\ge R_0\). Hence there exists \(C>0\) such that
\begin{equation}
\|x\|\ge R_0
\quad\Longrightarrow\quad
\|T(x)\|\le \|x\|-C.
\label{eq:SGD-map-diss}
\end{equation}
Moreover, since \(f\in C^1(\mathbb R^d)\), the map \(T\) is continuous and
therefore bounded on bounded sets.

\paragraph{Gaussian  or sub-Gamma noise perturbations.}

To fix a concrete example, assume now that the stochastic perturbation in \eqref{eq:SGD-random} is
Gaussian, namely
\[
\xi_n\sim \mathcal N(0,\sigma^2 I_d),
\qquad \sigma>0,
\]
independently and identically distributed. In this case the random map can
be written as
\[
T_\omega(x)=T(x)+\omega,
\qquad \omega\in\mathbb R^d,
\]
and the noise density is the Gaussian kernel
\[
\rho_\sigma(z)=\frac{1}{(2\pi\sigma^2)^{d/2}}
\exp\!\left(-\frac{\|z\|^2}{2\sigma^2}\right),
\]
which is \(C^\infty\), hence in particular uniformly \(C^1\). Our theory also extends to the sub-Gamma family of distributions, as detailed in the Appendix Section \ref{fosterlyapunov}. 

Therefore, if \eqref{eq:SGD-gradient-diss} holds, then Proposition \ref{prop:FL-exp-d} and Proposition \ref{prop:FL-exp-d-subgamma} apply to the present SGD
model, establishing an exponential
Foster-Lyapunov drift condition as required at Assumption (b).

\paragraph{Existence and uniqueness of the stationary density.}
In the Gaussian case, the Markov kernel of \eqref{eq:SGD-random} has the
strictly positive smooth density
\[
k(x,y)=\rho_\sigma(y-T(x)),
\qquad x,y\in\mathbb R^d.
\]
Hence the chain is irreducible and strong Feller. Combined with the
Foster--Lyapunov drift established above, this implies the existence of a
unique stationary probability measure \(\mu\). Moreover, \(\mu\) is
absolutely continuous with respect to Lebesgue measure, with density in the
relevant weighted space.

Hence all the assumptions of Theorem \ref{thm:gen2} are satisfied
for this stochastic gradient descent model.
This gives the following consequence:
\begin{corollary}[Exponential hitting-time statistics for stochastic gradient descent]
\label{cor:SGD-hitting-times}
Let \(f:\mathbb R^d\to\mathbb R\) be of class \(C^1\), and fix \(\delta>0\).
Consider the stochastic gradient descent model
\[
X_{n+1}=X_n-\delta \nabla f(X_n)+\xi_{n+1},
\]
where \((\xi_n)_{n\ge 1}\) is an i.i.d.\ sequence of Gaussian random vectors
with law \(\mathcal N(0,\sigma^2 I_d)\), \(\sigma>0\).

Assume that there exist constants \(C_0>0\) and \(R_0>0\) such that
\[
\langle \nabla f(x),x\rangle \ge C_0\|x\|
\qquad\text{for all }x\in\mathbb R^d\text{ with }\|x\|\ge R_0.
\]
Let \(\mu\) denote the unique stationary probability measure of the system,
and let \(x_0\in\mathbb R^d\) be such that the stationary density of \(\mu\)
is positive at \(x_0\). Let
\[
B_n:=B(x_0,r_n)
\]
be a sequence of balls with \(r_n\to 0\) such that
\[
n\,\mu(B_n)\longrightarrow \tau>0.
\]
Then
\[
\lim_{n\to\infty}
(\mathbb P\otimes\mu)\bigl(\tau_{B_n}>n\bigr)
=
e^{-\tau},
\]
where
\[
\tau_{B_n}(\omega,x):=\inf\{k\ge 1:\ X_k(\omega,x)\in B_n\}.
\]

Equivalently, for every \(t\ge 0\),
\[
\lim_{n\to\infty}
(\mathbb P\otimes\mu)\!\left(
\tau_{B_n}>\frac{t}{\mu(B_n)}
\right)
=
e^{-t}.
\]
In other words, the rescaled hitting times \(\mu(B_n)\tau_{B_n}\) converge
in distribution to an exponential random variable of parameter \(1\).
\end{corollary}

\section{Transfer operators, Banach spaces and regularization lemmas.}

In this section we set up the functional analytic framework to establish the main result of the paper.

\subsection{The transfer operators}

\label{sec:operators} In this section we recall the definition of the (annealed) transfer operator and of the (averaged) Koopman operator
associated to the system $(\ref{RS-new})$, showing some of the basic properties of these operators we will use later.

In Definition \ref{def:koopman-new} we defined the Koopman operator $P$ associated to a system as in  $(\ref{RS-new})$.

We remark that when the system is of the form \eqref{RS-q}, it also holds 
\begin{equation*}
(P\phi )(x)=\int_{\mathbb{R}^{d}}\phi (y)\hat{\rho}_{ T(x)}(y)dy.
\end{equation*}%

In this case we can also have a similar characterization for  the transfer operator ${L}:L^{1}(\mathbb{R}%
^{d})\rightarrow L^{1}(\mathbb{R}^{d})$. 
If $\nu $ is a Borel signed measure on $\mathbb{R}^{d}$ 
\begin{equation*}
\int_{\mathbb{R}^{d}}(P\phi )(x)d\nu (x)=\int_{\mathbb{R}^{d}}\int_{\mathbb{R%
}^{d}}\phi (y)\hat{\rho}_{T(x)}(y)dy~d\nu (x)
\end{equation*}%
supposing that $\nu $ has a density with respect to the Lebesgue measure $%
f\in L^{1}(\mathbb{R}^{d})$ i.e $d\nu =f(x)dx$ we can thus write 
\begin{equation}
\int_{\mathbb{R}^{d}}(P\phi )(x)d\nu (x)=\int_{\mathbb{R}^{d}}\phi
(y)\left( \int_{\mathbb{R}^{d}}\hat{\rho}_{T(x)}(y)f(x)dx\right) dy.
\label{duality}
\end{equation}%
We can then characterize the transfer operator associated to the system as

\begin{equation}
({L}f)(y):=\int \hat{\rho}_{T(x)}(y)f(x)dx.  \label{eq:stoctransfer}
\end{equation}

It is well known that the operator ${L}$ is a positive operator,
it preserves the integral and is a weak contraction with respect to the $%
L^{1}$ norm (see \cite{Gdisp} for details).

\noindent To prove  Theorem \ref{thm:gen2} we will use the construction outlined in \cite{kellerliverani2009} (see Section \ref{app:repfo}), similarly to what was done in \cite{OTHERPAPER} adapted to our case.
This is based on the idea of considering the target set $B_n$ as a hole in the phase space and consider it as an open system. The behavior of the resulting system is then studied  by means of the related transfer operators and their properties.

\begin{definition}\label{6}
Let \(x_0\in{\mathbb R}^d\). Let
\[
B_n=B(x_0,r_n)
\]
 be a sequence of
balls centred at \(x_0\).
We define the \textquotedblleft perturbed" versions of
the transfer operator by setting 
\begin{equation*}
({L}_{n}f)(x):=1_{B_{n}^{c}}(x)({L}f)(x),\ f\in L^{1}(%
\mathbb{R}^{d}).
\end{equation*}
\end{definition}

We define the perturbed Kolmogorov operator by duality as 
\begin{equation}
\int (P_{n}\phi )(x)f(x)dx\hspace{-3pt}=\int \phi (y)({L}_{n}f)(y)dy.
\label{mod}
\end{equation}

with $\phi \in L^{\infty }.$ For $P_{n}^{(n)}$ we have the following

\

\begin{proposition}
\label{p5}
For every $x \in \mathbb{R}^d$ and every $\varphi \in L^\infty(\mathbb{R}^d)$ we have
\[
\bigl[P_n^{(n)}(\varphi)\bigr](x)
= \mathbb{E}\Bigl(
\mathbf{1}_{B_n^c}(X_1(x)) \cdots \mathbf{1}_{B_n^c}(X_n(x))\,
\varphi\bigl(X_n(x)\bigr)
\Bigr),
\]
where $P_n^{(n)}$ denotes the $n$-th iterate of the perturbed Koopman operator $P_n$.
\end{proposition}

\begin{proof}
Recall that the (annealed) Koopman operator $P$ associated to the system  is
defined by
\[
(P\varphi)(x) = \mathbb{E}\bigl[\varphi(X_1(x))\bigr],
\]
and that the transfer operator $L$ is its dual: whenever $d\nu = f(x)\,dx$ with
$f\in L^1(\mathbb{R}^d)$, we have
\begin{equation}\label{eq:dual-L-P}
\int_{\mathbb{R}^d} (P\psi)(x) f(x)\,dx
=
\int_{\mathbb{R}^d} \psi(y)\,(Lf)(y)\,dy
\qquad\text{for all }\psi\in L^\infty(\mathbb{R}^d).
\end{equation}

For the open system with hole $B_n$, we defined the perturbed transfer operator by
\[
(L_n f)(y) := \mathbf{1}_{B_n^c}(y)\,(Lf)(y), \qquad f \in L^1(\mathbb{R}^d),
\]
and the perturbed Koopman operator $P_n$ by duality:
\begin{equation}\label{eq:def-Pn}
\int_{\mathbb{R}^d} (P_n\varphi)(x)\,f(x)\,dx
=
\int_{\mathbb{R}^d} \varphi(y)\,(L_n f)(y)\,dy
\qquad\text{for all } f \in L^1(\mathbb{R}^d).
\end{equation}

\medskip

\noindent\emph{Step 1: operator identity for $P_n$.}
Using the definition of $L_n$ and the duality relation \eqref{eq:dual-L-P}, we obtain
for every $f\in L^1(\mathbb{R}^d)$:
\begin{align*}
\int_{\mathbb{R}^d} (P_n\varphi)(x)\,f(x)\,dx
&=
\int_{\mathbb{R}^d} \varphi(y)\,\mathbf{1}_{B_n^c}(y)\,(Lf)(y)\,dy \\
&=
\int_{\mathbb{R}^d} \bigl(\mathbf{1}_{B_n^c}\,\varphi\bigr)(y)\,(Lf)(y)\,dy \\
&=
\int_{\mathbb{R}^d} \bigl(P(\mathbf{1}_{B_n^c}\varphi)\bigr)(x)\,f(x)\,dx,
\end{align*}
where in the last line we used \eqref{eq:dual-L-P} with $\psi=\mathbf{1}_{B_n^c}\varphi$.
Since this holds for all $f\in L^1(\mathbb{R}^d)$, we deduce the operator identity
\begin{equation}\label{eq:Pn-as-P}
P_n\varphi = P(\mathbf{1}_{B_n^c}\varphi)
\quad\text{a.e. on }\mathbb{R}^d.
\end{equation}
In particular we get a   probabilistic representation of $P_n$,
\begin{equation}\label{eq:Pn-1-step}
(P_n\varphi)(x)
=
\mathbb{E}\bigl[ \mathbf{1}_{B_n^c}(X_1(x))\,\varphi(X_1(x)) \bigr].
\end{equation}

\medskip

\noindent\emph{Step 2: induction on the number of iterates (operatorial form).}
We now prove by induction on $k\ge1$ that
\begin{equation}\label{eq:claim-k}
\bigl[P_n^{(k)}\varphi\bigr](x)
=
\mathbb{E}\Bigl(
\mathbf{1}_{B_n^c}(X_1(x)) \cdots \mathbf{1}_{B_n^c}(X_k(x))\,
\varphi\bigl(X_k(x)\bigr)
\Bigr).
\end{equation}

For $k=1$, \eqref{eq:claim-k} is exactly \eqref{eq:Pn-1-step}, so the claim holds.

Assume that \eqref{eq:claim-k} holds for some $k\ge1$. Using \eqref{eq:Pn-as-P}, we
can write, for every $\psi\in L^\infty$,
\[
P_n\psi = P\bigl(\mathbf{1}_{B_n^c}\psi\bigr).
\]
Applying this with $\psi = P_n^{(k)}\varphi$ and evaluating at $x$ gives
\begin{equation}\label{eq:Pn-k+1}
\bigl[P_n^{(k+1)}\varphi\bigr](x)
=
\bigl[P_n\bigl(P_n^{(k)}\varphi\bigr)\bigr](x)
=
\mathbb{E}\Bigl[
\mathbf{1}_{B_n^c}\bigl(X_1(x)\bigr)\,
\bigl(P_n^{(k)}\varphi\bigr)\bigl(X_1(x)\bigr)
\Bigr],
\end{equation}
where we used again the representation of $P$ in terms of the random orbit.

Now we apply the inductive identity \eqref{eq:claim-k} with initial point
replaced by $X_1(x)$. Since $X_j(\cdot)$ is defined by iterating the same map
\textup{(4)}, the orbit starting from $X_1(x)$ at step~1 is given by
$X_2(x),\dots,X_{k+1}(x)$. Thus
\[
\bigl(P_n^{(k)}\varphi\bigr)\bigl(X_1(x)\bigr)
=
\mathbb{E}\Bigl(
\mathbf{1}_{B_n^c}(X_2(x)) \cdots \mathbf{1}_{B_n^c}(X_{k+1}(x))\,
\varphi\bigl(X_{k+1}(x)\bigr)
\Bigr).
\]
Substituting this expression into \eqref{eq:Pn-k+1} and using the linearity of the
expectation, we obtain
\[
\bigl[P_n^{(k+1)}\varphi\bigr](x)
=
\mathbb{E}\Bigl(
\mathbf{1}_{B_n^c}(X_1(x)) \cdots \mathbf{1}_{B_n^c}(X_{k+1}(x))\,
\varphi\bigl(X_{k+1}(x)\bigr)
\Bigr),
\]
which is precisely \eqref{eq:claim-k} with $k$ replaced by $k+1$. This completes
the induction.

\medskip

\noindent Taking $k=n$ in \eqref{eq:claim-k} yields
\[
\bigl[P_n^{(n)}\varphi\bigr](x)
=
\mathbb{E}\Bigl(
\mathbf{1}_{B_n^c}(X_1(x)) \cdots \mathbf{1}_{B_n^c}(X_n(x))\,
\varphi\bigl(X_n(x)\bigr)
\Bigr),
\]
which is exactly the statement of the proposition.
\end{proof}

\subsection{Functional spaces adapted to the system}

\label{sec:spacesRD}

We now define suitable functional spaces on which the transfer operators
introduced in the previous sections have a regularizing behavior and nice
spectral properties. 
To deal with the non compactness of the phase space, we will use some weighted spaces very similar to the ones used in \cite{OTHERPAPER} in the context of SDE, adapting it to the different set up ne need to deal here.

These spaces are in some sense spaces of weighted Bounded Variation of $L^1$ functions with exponential growing weights (while in \cite{OTHERPAPER} quadratic weights were used).
 The spaces will be
denoted as $BV_{A,\lambda }$,  and $L_{A,\lambda }^{1}$.

Let $\lambda
>0$ considered at Assmption $(c)$; define the weight functions 
\begin{eqnarray}
\rho _{A,\lambda }\left(\Vert x\Vert |\right)  &=&Ae^{\lambda
||x||}  \label{WF} 
\end{eqnarray}

where $A$ is such that $\forall x$ $\rho _{A,\lambda
}\left( |\left\vert x\right\vert |\right) \geq 1+\left\vert x\right\vert ^{2}$. 

From now on, we will fix such a $\lambda $ and $A$ and denote $\rho :=\rho
_{A,\lambda }$.

Let $L_{A,\lambda }^{1}
$ be the space of
Lebesgue measurable $f:\mathbb{R}^{d}\rightarrow \mathbb{R}$ such that 
\begin{equation*}
\left\Vert f\right\Vert _{L_{A,\lambda }^{1}
}:=\int_{\mathbb{R}^{d}}\rho \left( \left\vert x\right\vert \right)
\left\vert f\left( x\right) \right\vert dx<\infty .
\end{equation*}%
Note that, $L_{A,\lambda }^{1}\subset L^{1}$ and if $f\in
L_{{A,\lambda }}^{1}$ then $\Vert f\Vert _{L^{1}} \leq
||f||_{L_{A,\lambda }^{1}}$.

We now define a sort of Bounded Variation regularity following the idea of \cite{keller1985} and  \cite{saussol1998}.
  For a Borel subset $S\subseteq 
\mathbb{R}^{d}$ let us define 
\begin{equation*}
osc(f,S)={ess}_{x\in S}f-{ess}_{x\in S}f.
\end{equation*}%
We now define the seminorm:%
\begin{equation*}
|f| _{osc}=\sup_{0<\epsilon \leq 1}\epsilon
^{-1}\int_{\mathbb{R}^{d}}osc(f,B_{\epsilon }(x))d\psi (x).
\end{equation*}

Here the measure $\psi$ is a Radon probability measure on $\mathbb{R}^d$ and
we require that $\psi$ is absolutely continuous with respect to
the Lebesgue measure, having a continuous bounded density $\psi^{\prime }$
such that $\psi^{\prime }>0$ everywhere.

We can then define  the $\Vert \cdot \Vert
_{BV_{A,\lambda }}$ norm by setting 
\begin{eqnarray}
\left\Vert f\right\Vert _{BV_{A,\lambda }} &:&=\left\Vert f\right\Vert
_{L_{A,\lambda }^{1}}+| f|_{osc}.
\end{eqnarray}%
 $\Vert \cdot
\Vert _{BV_{A,\lambda }}$  defines a norm and the set of functions for which this norm is finite is a Banach space which we denote $%
BV_{A,\lambda }.$

We prove that our weighted bounded variation space $BV_{A,\lambda }\left( 
\mathbb{R}^{d}\right) $ is compactly immersed in $L^{1}.$

\begin{theorem}
\label{thm:embeddingRD} $BV_{A,\lambda }\left( \mathbb{R}^{d}\right)
\hookrightarrow L^{1}\left( \mathbb{R}^{d}\right) $ is a compact embedding.
\end{theorem}

\begin{proof}
We recall the definition of a certain space used in \cite{OTHERPAPER}, characterized by the norm $\Vert \cdot \Vert _{BV_{2}}$, we are  going to define.
Let us consider the weight function $\rho _{2}\left( |\left\vert x\right\vert |\right)  =1+\left\vert x\right\vert ^{2}$.
Let $L_{2 }^{1}$ be the space of Lebesgue
measurable $f:\mathbb{R}^{d}\rightarrow \mathbb{R}$ such that 
\begin{equation*}
\left\Vert f\right\Vert _{L_{2}^{1}\left( \mathbb{R}^{d}\right) }:=\int_{%
\mathbb{R}^{d}}\rho _{2}\left( \left\vert x\right\vert \right) \left\vert
f\left( x\right) \right\vert dx<\infty .
\end{equation*}%

Let us consider the associated Bounded Variation norm
$$\left\Vert f\right\Vert _{BV_{2}}=\left\Vert f\right\Vert_{L_{2}^{1}} +\Vert f\Vert _{osc}.$$
The set of $L^1$ densities for which this norm is finite is a complete normed  vector space (see \cite{OTHERPAPER}) we will denote by $BV_2$.

Since $\rho \geq \rho _{2}$ by the choice we made on $A$ above, note that, $L_{A,\lambda }^{1}\subset L_{2}^{1}\subset L^{1}$ and if $f\in
L_{2}^{1}$ then $\Vert f\Vert _{L^{1}}\leq \Vert f\Vert _{L_{2}^{1}}\leq
||f||_{L_{A,\lambda }^{1}}$.

In \cite{OTHERPAPER} it is proved that there is a compact embedding $%
BV_{2}\left( \mathbb{R}^{d}\right) \hookrightarrow L^{1}\left( \mathbb{R}%
^{d}\right) $. 

Since $\rho \geq \rho _{2}$ we have that the unit ball of $BV_{A,\lambda
}\left( \mathbb{R}^{d}\right) $ is a closed subset of the unit ball of $%
BV_{2}\left( \mathbb{R}^{d}\right) $ and hence totally bounded  in $%
L^{1}\left( \mathbb{R}^{d}\right) $. 
\end{proof}

Let us suppose now that $f\in BV_{A,\lambda }$ and let $\mathcal{K}$ be a
compact set in $\mathbb{R}^{d}.$ We will need later on an estimate of the $%
L^{\infty }$ norm of $f$ on $\mathcal{K}$, designated as $
||f||_{L^{\infty }(\mathcal{K})}.$

\begin{proposition}
\label{imm} For any compact set $\mathcal{K}\subset \mathbb{R}^d$ we have

\begin{equation}  \label{in}
\|f\|_{L^{\infty}( \mathcal{K})}\le \frac{\max\left(||\psi^{\prime
}||_{L^{\infty}(\mathbb{R}^d)},1\right)}{d_{\mathcal{K}}}||f||_{BV_{A,\lambda}},
\end{equation}
where again $\psi^{\prime }$ denotes the density of $\psi$ with respect to
the Lebesgue measure and $d_{\mathcal{K}}=\text{essinf}_{x\in \mathcal{K}%
}\psi(B_1(x))>0$. 
\end{proposition}

\begin{proof}
    The proposition  is proved in  \cite{OTHERPAPER}, Proposition 17 with the $BV_2$ norm in the place of the $BV_{A,\lambda}$ norm  (see the proof of Theorem \ref{thm:embeddingRD} for the definition of the $BV_2$ norm). However by the choice of $A$ made above we have  $||f||_{BV_{2}}\leq ||f||_{BV_{A,\lambda}}$ proving the statement.
\end{proof}

\subsection{Regularization properties for the transfer operator.} \label{sec:spacerdtotd}

In this section we collect some estimates about the regularization properties of the transfer operator.

\begin{proposition}
    Under the standing assumptions (a),...,(c),
     there exist \(a\in(0,1)\) \(C_1,C_2, C_{\mathrm{reg}}>0\) such
that

\begin{equation}
|{L} f|_{osc}
\le
C_{\mathrm{reg}}\,\|f\|_{L^1}
\qquad\text{for every }f\in L^1,
\label{1eq:reg-q-step}
\end{equation}

and\begin{equation}
\|{L} f\|_{BV_{A,\lambda_0}}
\le
\alpha \,\|f\|_{BV_{A,\lambda_0}}
+
C_{2}\,\|f\|_{L^1}
\qquad\text{for every }f\in BV_{A,\lambda_0}.
\label{LY-unperturbed}
\end{equation}

\end{proposition}

\begin{proof}
We first prove the regularising estimate
\eqref{1eq:reg-q-step}. By Assumption \textup{(a)}, for each
\(z\in{\mathbb R}^d\), the law of the  noise has a density
\(\hat\rho_z\in C^1({\mathbb R}^d)\), and
\[
M_{reg}:=\sup_{z\in{\mathbb R}^d}\|\hat\rho_z\|_{C^1}<\infty.
\]

We know  the transfer operator associated to the system can be rewitten as

\begin{equation}
({L}f)(y):=\int \hat{\rho}_{T(x)}(y)f(x)dx.  \label{eq:stoctransfer}
\end{equation}

This shows that \begin{equation}\label{C^1}
||{L}f||_{C^1}\leq M_{reg}||f||_{L^1}
\end{equation}implying 
\eqref{1eq:reg-q-step}.

We now prove  \eqref{LY-unperturbed}. By Assumption
\textup{(b)},
\[
P V
\le
\kappa V+b\,\mathbf 1_K.
\]
Using the duality between \({L}\) and
\({P}\), for every nonnegative \(f\in L^1_{A,\lambda_0}\),
\[
\|{L}f\|_{L^1_{A,\lambda_0}}
=
\int_{{\mathbb R}^d}({L}f)(z)\,V(z)\,dz
=
\int_{{\mathbb R}^d}f(z)\,({P}V)(z)\,dz.
\]
Hence
\[
\|{L}f\|_{L^1_{A,\lambda_0}}
\le
\kappa \int_{{\mathbb R}^d}f(z)V(z)\,dz
+
b\int_K f(z)\,dz,
\]
that is,
\[
\|{L}f\|_{L^1_{A,\lambda_0}}
\le
\kappa \|f\|_{L^1_{A,\lambda_0}}+b\,\|f\|_{L^1(K)}.
\]
By positivity of \({L}\), the same estimate extends to
signed \(f\) by replacing \(f\) with \(|f|\). This proves \[
\|{L}f\|_{L^1_{A,\lambda_0}}
\le
\kappa \|f\|_{L^1_{A,\lambda_0}}+b\,\|f\|_{L^1}
\]
which combined with \eqref{1eq:reg-q-step}, gives \eqref{LY-unperturbed}.\end{proof}

\subsection{Regularization for the perturbed operators}

In this section we prove Lasota Yorke inequalities, like in Proposition \ref{LY-unperturbed} for the perturbed operators  $L_n$.

Let
\[
|f|_{osc}
:=
\sup_{0<\varepsilon\le \varepsilon_0}
\varepsilon^{-1}
\int_{{\mathbb R}^d}
\operatorname{osc}(f,B_\varepsilon(z))\,d\psi(z)
\]
denote the oscillation seminorm, where
\[
\operatorname{osc}(f,B_\varepsilon(z))
:=
\operatorname*{ess\,sup}_{B_\varepsilon(z)} f
-
\operatorname*{ess\,inf}_{B_\varepsilon(z)} f.
\]

In the following we will denote by:
\[
\|\cdot\|_s:=|\cdot |_{osc}+ \|\cdot\|_{L^1},
\qquad
\|\cdot\|_w:=\|\cdot\|_{L^1},
\]

Next Lemma shows perturbative estimates on the size of the perturbation which is made by adding a hole to the system, based on the above norms and seminorms. 

\begin{lemma}[Oscillation estimate for cylindrical holes]
\label{lem:cylinder-hole-osc}
 Let
us consider the sequance of targets \(
B_n=B(x_0,r_n)
\)
with \(r_n\to 0\).

For a measurable bounded function \(g:{\mathbb R}^d\to\mathbb R\), define
\[
M_n g:=\mathbf 1_{B_n^c}g,
\qquad
H_n g:=\mathbf 1_{B_n}g.
\]

Then, for every bounded \(g\),
\begin{equation}
|M_n g|_{osc}
\le
|g|_{osc}
+
2\,\Delta_n\,\|g\|_{L^\infty},
\label{eq:Mn-osc-est}
\end{equation}
and
\begin{equation}
|H_n g|_{osc}
\le
|g|_{osc}
+
2\,\Delta_n\,\|g\|_{L^\infty},
\label{eq:Hn-osc-est}
\end{equation}
where
\begin{equation}
\Delta_n
:=
\sup_{0<\varepsilon\le \varepsilon_0}
\varepsilon^{-1}
\psi\!\left(
U_\varepsilon(\partial B_n)
\right),
\qquad
U_\varepsilon(\partial B_n)
:=
\{z\in{\mathbb R}^d:\operatorname{dist}(z,\partial B_n)<\varepsilon\}.
\label{eq:Delta-n-def}
\end{equation}

\end{lemma}

\begin{proof}
We prove \eqref{eq:Mn-osc-est}; the proof of \eqref{eq:Hn-osc-est} is
identical.

We recall the definition of oscillation seminorm and its notations
\begin{equation*}
|f| _{osc}=\sup_{0<\epsilon \leq 1}\epsilon
^{-1}\int_{\mathbb{R}^{d}}osc(f,B_{\epsilon }(x))d\psi (x).
\end{equation*}

Fix \(z\in{\mathbb R}^d\) and \(\varepsilon>0\). We distinguish two cases.

\medskip
\noindent
\emph{Case 1: \(B_\varepsilon(z)\cap \partial B_n=\varnothing\).}
Then the ball \(B_\varepsilon(z)\) is either entirely contained in \(B_n\) or
entirely contained in \(B_n^c\). In the first case, \(M_n g\equiv 0\) on
\(B_\varepsilon(z)\), so
\[
\operatorname{osc}(M_n g,B_\varepsilon(z))=0
\le \operatorname{osc}(g,B_\varepsilon(z)).
\]
In the second case, \(M_n g=g\) on \(B_\varepsilon(z)\), hence
\[
\operatorname{osc}(M_n g,B_\varepsilon(z))
=
\operatorname{osc}(g,B_\varepsilon(z)).
\]

\medskip
\noindent
\emph{Case 2: \(B_\varepsilon(z)\cap \partial B_n\neq \varnothing\).}
In this case the ball intersects both \(B_n\) and \(B_n^c\), so
\[
\operatorname{osc}(M_n g,B_\varepsilon(z))
\le
\operatorname{osc}(g,B_\varepsilon(z))
+
2\|g\|_{L^\infty}.
\]
Indeed, multiplying by \(\mathbf 1_{B_n^c}\) can only create an additional
jump between the values of \(g\) and \(0\), whose size is bounded by
\(2\|g\|_{L^\infty}\).

Combining the two cases, we obtain for every \(z\) and \(\varepsilon\),
\[
\operatorname{osc}(M_n g,B_\varepsilon(z))
\le
\operatorname{osc}(g,B_\varepsilon(z))
+
2\|g\|_{L^\infty}\,
\mathbf 1_{U_\varepsilon(\partial B_n)}(z).
\]
Integrating with respect to \(d\psi(z)\), multiplying by
\(\varepsilon^{-1}\), and taking the supremum over
\(0<\varepsilon\le \varepsilon_0\), we get
\[
|M_n g|_{osc}
\le
|g|_{osc}
+
2\|g\|_{L^\infty}
\sup_{0<\varepsilon\le \varepsilon_0}
\varepsilon^{-1}
\psi\!\left(U_\varepsilon(\partial B_n)\right),
\]
which is exactly \eqref{eq:Mn-osc-est}.
\end{proof}

\begin{lemma}[Lasota--Yorke inequality for the perturbed operators]
\label{lem:LY-punctured-q1}
Assume that the random dynamical system \eqref{RS-new} satisfies
Assumption \textup{(a),(b)}.
Let
\[
B_n:=B(z_0,r_n)\subset{\mathbb R}^d,
\qquad
M_n f:=\mathbf 1_{B_n^c}f,
\qquad
L_n:=M_nL.
\]

Then there exist \(N\in\mathbb N\) and a constant \(C_{\mathrm{LY}}>0\) such
that, for every \(n\ge N\) and every \(f\in BV_{A,\lambda}\),
\begin{equation}
\|L_n f\|_{BV_{A,\lambda}}
\le
\kappa \|f\|_{BV_{A,\lambda}}
+
C_{\mathrm{LY}}\,\|f\|_{L^1_{A,\lambda}}.
\label{eq:LY-punctured-q1}
\end{equation}
In particular, the family \((L_n)_{n\ge N}\) satisfies a uniform
Lasota--Yorke inequality on \(BV_{A,\lambda}\).
\end{lemma}

\begin{proof}
By Lemma \ref{lem:cylinder-hole-osc}, applied to \(g=L f\), and the positivity of the involved operators
\[
||L_n f||_{BV_{A,\lambda}}
=
||M_n(L f)||_{BV_{A,\lambda}}
\le
||L f||_{BV_{A,\lambda}}+2\,\Delta_n\,\|Lf\|_{L^\infty}.
\]
Using \eqref{LY-unperturbed} and \ref{C^1}
 we get
\begin{equation}
\|{L} f\|_{BV_{A,\lambda_0}}
\le
\alpha \,\|f\|_{BV_{A,\lambda_0}}
+ C_{2}\,\|f\|_{L^1}+ M_{reg}||f||_{L^1}
\end{equation}
{for every }$f\in BV_{A,\lambda_0}$.
\end{proof}

\section{The response of the leading eigenvalue after suitable small perturbations.}

\label{app:repfo}

Our main technical tool for the proof of Theorem \ref{thm:gen2} is a 
result  due to Keller and
Liverani \cite{kellerliverani2009} 
on the response of the leading eigenvalue of the transfer operator after small suitable perturbations.
In this  section we recall the result and we verify that the required assumptions hold in the class of systems we consider.

We
consider a Banach space $(\mathcal{B},\Vert \cdot \Vert )$ and we denote
with $\mathcal{B}^{\ast }$ its dual. Then let ${L}_{\epsilon }:%
\mathcal{B}\rightarrow \mathcal{B}$ be a family of uniformly bounded linear
operators, where $\epsilon \in E$, and $E$ is the interval $E=(0,\overline{\epsilon }]$ for some $\overline{\epsilon }>0.$
We suppose the following assmptions hold for the family ${L}_{\epsilon }:$

\begin{itemize}
\item[R1] The operators ${L}_{\epsilon},$ $\epsilon \in E,$ must
satisfy the spectral decomposition 
\begin{equation}  \label{specdec}
\lambda_{\epsilon}^{-1}{L}_{\epsilon}=\varphi _{\epsilon}\otimes \nu
_{\epsilon}+Q_{\epsilon}
\end{equation}
where $\lambda_{\epsilon}\in \mathbb{C},$ $\varphi _{\epsilon}\in \mathcal{B}
,\nu_{\epsilon}\in \mathcal{B}^{*},$ $Q_{\epsilon}: \mathcal{B} \rightarrow 
\mathcal{B}$ is a linear operator verifying 
\begin{equation}  \label{qqcc}
\sum_{n=0}^{\infty}\sup_{\epsilon\in E}\|Q_{\epsilon}^n\|<\infty.
\end{equation}

Moreover, ${L}_{\epsilon}\varphi_{\epsilon}=\lambda_{\epsilon}%
{\varphi_{\epsilon}},$ $\nu_{\epsilon}{L}_{\epsilon}=\lambda_{\epsilon}\nu_{\epsilon},$ $\nu_{\epsilon}(\varphi_{\epsilon})=1,$ $%
\nu_{\epsilon}Q_{\epsilon}=0,$ $Q_{\epsilon}(\varphi_{\epsilon})=0.$\label%
{eq:B5}

We also require that $\nu_{0}(\varphi_{\epsilon})=1$ and 
\begin{equation}  \label{coon}
\sup_{\epsilon \in E}\|\varphi_{\epsilon}\|<\infty.
\end{equation}
\end{itemize}

\begin{itemize}
\item[R2] When $\varepsilon $ is small, ${L}_{\varepsilon }$ is a
small perturbation of ${L}_{0}$, in the following sense:

\begin{equation*}
\pi_{\epsilon}:=\sup_{f\in \mathcal{B}, \|f\|\le 1}|\nu_0(({L}_0-%
{L}_{\epsilon})(f))|\rightarrow 0, \ \epsilon\rightarrow 0.
\end{equation*}
\end{itemize}

\begin{itemize}
\item[R3] We now set 
\begin{equation*}
\Delta_{\epsilon}:=\nu_0(({L}_0-{L}_{\epsilon})(\varphi_0)).
\end{equation*}
Then we require that there exists $C_0\geq 0 $ such that
\begin{equation*}
\pi_{\epsilon}\|({L}_0-{L}_{\epsilon})\varphi_0\|\le \ C_0 \ |\Delta_{\epsilon}|.
\end{equation*}

\item[R4] Let us consider the  quantities $q_{k,\varepsilon }$ defined by
\begin{equation*}
q_{k,\varepsilon }:=\frac{\nu _{0}(({L}_{0}-{L}_{\varepsilon
}){L}_{\varepsilon }^{k}({L}_{0}-{L}_{\varepsilon
})(\varphi_0))}{ \Delta_{\epsilon}}.
\end{equation*}
We will assume that for each $k\geq 0$ the following limit exists 
\begin{equation}  \label{eee}
\lim_{\epsilon \rightarrow 0}q_{k,\varepsilon }=q_k,
\end{equation}
and we define the \emph{extremal index of the system} as 
\begin{equation}  \label{extremal}
\theta=1-\sum_{k=0}^{\infty}q_k.
\end{equation}
\end{itemize}
Under these assumptions, the main result of \cite{kellerliverani2009} shows that
\begin{proposition}
\label{thm:repfo} If $R1-R4$ are satisfied then 
\begin{equation}  \label{eq:repfoexp}
\lambda_{\epsilon} = 1 - \theta \Delta_{\epsilon} + o( \Delta_{\epsilon}).
\end{equation}
\end{proposition}

\subsection{Verifying the assumptions in our case}\label{monu}

In this section we show that the assumptions of Proposition \ref{thm:repfo} are verified in our case, where the Banach space considered is $BV_{A,\lambda}$, the unperturbed operator is $L$ and perturbed operators are the operators $L_n$.

\medskip

\noindent \textbf{Assumption R1}. 
The sought for spectral decomposition for the perturbed and unperturbed operators will be provided by  the classical result of \cite{kellerliverani1999}.
In this context we will consider $BV_{A,\lambda}$ and $L^1$ as a strong and weak space. To apply the approach of \cite{kellerliverani1999} we need that a uniform Lasota Yorke inequality be satisfied. This is proved in Proposition \ref{1eq:reg-q-step} and Lemma \ref{lem:LY-punctured-q1}.

A compact immersion of the strong space into the weak one is also needed, and this is  proved at Theorem \ref{thm:embeddingRD}. 
In this context, the results of \cite{kellerliverani1999}  provide the stability of the spectral picture when the operator perturbation is small in a mixed norm.
More precisely, we have to verify that
\begin{equation}  \label{vbn}
||{L}-{L}_{n}||_{BV_{A,\lambda}\to L^1}\rightarrow
0, as \  n\rightarrow \infty.
\end{equation}
Indeed we have
\begin{equation}\label{biv} ||(L_{n}-L)f||_{L^1}=\|1_{B_{n}}({L} f)\|_{L^{1}} \leq \text{Leb}(B_n)||{L} f||_\infty.\leq C_2 \text{Leb}(B_n)||f||_1.  
\end{equation}
This allows us to apply the main results of \cite{kellerliverani1999}. 
The spectral decomposition \eqref{specdec} then follows from Theorem 1 of \cite{kellerliverani1999}. 
We remark  that in our case $\lambda_0=1$ is a simple eigenvalue of ${L}$ because of the uniqueness assumption e) on the stationary measure $\mu$. 

By Corollary 2, part 2 in \cite{kellerliverani1999}, we can deduce (\ref{qqcc}).
Finally \eqref{coon}  follows from the uniform Lasota-Yorke inequality
proved in Lemma \ref{lem:LY-punctured-q1}.

\medskip

\noindent \textbf{Assumption R2}. 
The transfer operator $ L$ preserves the integral, then  $\nu_0$ in our case is the integral with respect to the Lebesgue measure. Since the integral of a function is bounded by its  $L^1$ norm, hence the computation done at \eqref{biv} also establishes R2.

\medskip

\noindent In our setting, \textbf{Assumption R3} can be stated in the
following way:
 Let $f_{0}$ denote the density of the stationary
measure $\mu $.
By Proposition \ref{lem:LY-punctured-q1} the quantity $\Vert ({L}%
-{L}_{n})f_{0}\Vert _{BV_{A,\lambda }}=\Vert 1_{B_{n}}{L}f_{0}\Vert _{BV_{A,\lambda }}$ is bounded by a constant $\tilde{K}.$
We saw in \eqref{biv} that $\pi _{n} \leq C \text{Leb}(B_{n})$ for some $C\geq 0$.
Since we assumed $f_0(x_0)>0$, and $f_0\in C^1$ we also get that $\Delta_\epsilon \geq  C_2\text{Leb}(B_{n})$, for some $C_2\geq 0$ thus there is $C'\geq 0 $ such that

\begin{equation*}
\frac{\pi _{n}\Vert ({L}-{L}_{n})f_{0}\Vert _{BV_{A,\lambda
}}}{\Delta_\epsilon}\leq C^{\prime }.
\end{equation*}

\medskip

\noindent \textbf{Assumption R4} needs some more work. The next proposition
will show that all the quantities $q_k$ defined in the Assumption R4 are
equal to $0.$ In the present setting they are defined as the limit for $%
n\rightarrow \infty$ of the following quantities: 
\begin{equation*}
q_{k,n}=\frac{\int ({L}-{L}_{n}){L}^k_{n}(%
{L} -{L}_{n})(f_0)dm }{\mu(B_n)}
\end{equation*}
where $m$ is the Lebesgue measure on $\mathbb{R}^d$.

\begin{proposition}
\label{prop:R7}

For each $k\geq 0$ we have 
\begin{equation*}
\lim_{n \rightarrow \infty}q_{k,n }=0.
\end{equation*}
\end{proposition}

\begin{proof}
Let us consider $f_{k,n}:=\frac{{L}^k_{n}({L}-{L}_{n})(f_0)}{\mu(B_n)};$ we have
$$
\lim_{n\rightarrow \infty}q_{k,n } = \lim_{n
\rightarrow \infty}\int ({L}-L_{n })f_{k,n } dm
= \lim_{n \rightarrow \infty}\int 1_{B_{n }}{L} f_{k,n }dm\le \lim_{n \rightarrow \infty}m(B_n)\|{L} f_{k,n}\|_{\infty}.
$$  
Again, by \eqref{C^1}, the $C^1$ norm of ${L}f_{k,n},$ and therefore its infinity norm, is bounded by $M_{reg}\|f_{k,n}\|_{L^1}$. 
  On the other hand 
\[\|f_{k,n}\|_{L^1}=\int f_{k,n} dm \leq \frac{1}{\mu(B_n)}\int_{B_n}f_0dm=1. \]
Then we get $\lim_{n\rightarrow \infty}q_{k,n }= 0.$ \end{proof}

\subsection{Proof of Theorem \protect\ref{thm:gen2}}

\label{Pgen2}\label{app:repfo2} In this section we can collect all the
previous estimates and finally prove the main result of the paper.

\begin{proof}[Proof of Theorem \protect\ref{thm:gen2}]
The proof is similar to the proof of Theorem 1 in \cite{OTHERPAPER}, but with some adaptation to our case. For this reason we report it completely.
We rewrite the main formula in left hand of \eqref{eq:main-hitting-law}, using the notation above as {\ 
\begin{equation}  \label{eq:fromPtoE}
\begin{aligned} & \mathbb{P}\otimes \mu \left( X_{{k}}(x)\in B_{n}^{c}\text{
for every }k=0,\ldots,n -1 \right) \\ &= \int_{\Omega \times \mathbb{R}^D}
1_{B_{n}^{c}} \left(X_{{0}} (x,\omega)\right) \cdots 1_{B_{n}^{c}} \left(
X_{{n-1}}(x,\omega)\right) d\omega d\mu(x) \\ & = \int_{\mathbb{R}^D}
\bE\left[ 1_{B_{n}^{c}} \left(X_{{0}}(x) \right) \cdots 1_{B_{n}^{c}}
\left( X_{{n-1}}(x) \right) \right] d\mu(x) . \end{aligned}
\end{equation}
}


By using Proposition \ref{p5} and (\ref{mod}), 
recalling that the invariant measure $\mu$ has density $f_{0}$, we can rewrite the last equation in operator-like way as 
\begin{equation}  \label{pe}
\begin{split}
& \int_{\mathbb{R}^D} \bE\left[ 1_{B_{n}^{c}} (X_{{0}}(x) ) \cdots 1_{B_{n}^{c}}
(X_{{n-1}}(x) )1(X_{{n-1}}(x) ) \right] f_0(x) dx \\
& = \int \left(
{ P}^n_{n }\right)(1)(x)f_0(x)dx= \int {L}_{n}^{n}f_0dx
\end{split}%
\end{equation}
where $P$ is the Koopman operator as in Definition \ref{def:koopman-new} and ${ L_n}$ is the perturbed operator as in Defition \ref{6}. 
Now we apply Proposition \ref{thm:repfo} to ${ L}$
and to its perturbations ${ L}_{n}$. We consider as a strong space the
space $\cB = BV_{A,\lambda}$ and $L^1$ as a weak space. The
assumptions to apply Proposition \ref{thm:repfo} are verified  in section \ref{monu}; therefore we have the following spectral decomposition
for the  operators ${L}_{n}:$
\begin{equation}  \label{eq:decomposition}
\lambda_{n}^{-1} {L}_{n} = f_{n} \otimes \mu_{n} + Q_{n}
\end{equation}
where $f_{n}\in BV_{\lambda,A } , \ \mu_{n} \in (BV_{\lambda, A} )^{\prime }$ and $Q_{n} :
BV_{\lambda,A} \to BV_{\lambda,A}$ is a bounded operator with spectral radius uniformly bounded
as $n$ varies by some $\rho<1.$ Moreover ${L}_{n}f_{n}=\lambda_{n}f_{n}$
and $\mu_{n}{L}_{n}=\lambda_{n}\mu_{n}.$\newline

By denoting with $\langle \mu_{n},g\rangle$ the action of the linear
functional $\mu_{n}$ over $g\in BV_{A,\lambda},$ we have
\begin{equation}  \label{eq:perturbation}
{L}_{n} g=\lambda_{n} f_{n} \langle \mu_{n},g\rangle + \lambda_{n}
Q_{n} (g);
\end{equation}
 we recall the normalization $\int f_{n} dx=1$ and $\langle
\mu_{n},f_{n}\rangle=1.$

We now use \eqref{eq:perturbation} to estimate $\int ({L}_{n}^{n} f_0)(x) dx $. Since we have a direct sum
decomposition for the operator, we can iterate and get 
\begin{equation}  \label{eq:naction}
\int ({L}_{n}^{n} f_0)(x) \ dx =\lambda^{n}_{n} \langle
\mu_{n},f_0\rangle + \lambda^{n}_{n}\int (Q^n_{n} f_0)(x)\ dx.
\end{equation}

Now, $1$ is the largest unique eigenvalue of the unperturbed operator ${L}$. By Proposition\ref{prop:R7} we have $%
\theta=1.$
Applying Proposition\ref{thm:repfo} we get: 
\begin{equation}  \label{eq:leadingperturb}
\lambda_n=1- \mu(B_n)+o(\mu(B_n)).
\end{equation}

Then by substituting \eqref{eq:leadingperturb} in \eqref{eq:naction} we have 
\begin{equation}  \label{eq:perturbation3}
\begin{split}
\int ({L}_{n}^{n} f_0)(x) \ dx & =e^{n\log(1- \mu(B_n)+o(\mu(B_n)))}\left[%
\langle \mu_{h,n},f_0\rangle + \int Q^n_{n}f_0\ dx \right] \\
& =e^{-n \mu(B_n)+no(\mu(B_n))}\left[\langle \mu_n,f_0\rangle + \int
Q^n_{n} f_0\ dx\right].
\end{split}%
\end{equation}

Let us now recall that by \cite[Lemma 6.1]{kellerliverani1999} we have $\langle
\mu_n,f_0\rangle \rightarrow 1.$ Moreover by \eqref{qqcc} we get a uniform
exponential convergence to zero of 
\begin{equation*}
\int Q^n_{n} f_0\ dx \leq \|Q^n_{n} f_0\|_{L^1}\le \|Q^n_{n}
f_0\|_{BV_2} \le \text{Const} \ \rho^n.
\end{equation*}
Now, as assumed in the statement of Theorem \ref{thm:gen2}, we choose
the sequence $\{ u_n \}_{n \in \bN}$ and a $\tau \in \bR, \tau > 0$ such
that 
\begin{equation}  \label{eq:hyplimit}
n\, \mu(B_n)\rightarrow \tau.
\end{equation}
Thus 
\begin{equation}  \label{eq:stop}
\int {L}_{n}^n f_0 \ dx\rightarrow e^{-\tau}
\end{equation}
proving the claim. 
\end{proof}

\section{Appendix: Dissipative dynamics for Gaussian and subGamma noises in $\mathbb{R}^d$}
\label{fosterlyapunov}

In this appendix we consider random perturbations of suitably dissipative maps on \(\mathbb R^d\) by additive  noise. We prove that the dissipativity
of the deterministic dynamics implies an exponential Foster--Lyapunov condition, as required to establish the assumption (b) in our main framework.

Let $X=\mathbb{R}^d$; 
let $T:\mathbb{R}^d\to\mathbb{R}^d$ be a measurable map.
Assume that $T$ is dissipative in the following sense:
\begin{equation}\label{eq:dissip-d}
\exists\,R_0>0,\ C>0\ \ \text{such that}\ \ \|x\|\ge R_0\ \Longrightarrow\ \|T(x)\|\le \|x\|-C.
\end{equation}
Assume also local boundedness:
\begin{equation}\label{eq:local-bdd-d}
\forall R>0,\qquad \sup_{\|x\|\le R} \|T(x)\|<\infty.
\end{equation}
Let $(\xi_n)_{n\ge 1}$ be an i.i.d.\ sequence of random vectors.

\medskip

Let us consider the following random iteration 
\[
X_{n+1}=T(X_n)+\xi_{n+1}.
\]

For bounded measurable observables $\varphi:\mathbb{R}^d\to\mathbb{R}$, the annealed Koopman operator in this case can be written as
\begin{equation}\label{eq:P-def-d}
(P\varphi)(x)
:=\mathbb{E}\big[\varphi(T(x)+\xi)\big]
=\int_{\mathbb{R}^d} \varphi(T(x)+z)\,\rho(z)\,dz,
\end{equation}
where
$\rho$ is the density of $\xi$.

The annealed transfer operator on densities, in this case can be written as follows.
Let $m$ be Lebesgue measure on $\mathbb{R}^d$. The Markov kernel associated to \eqref{eq:P-def-d}
has transition density
\[
k(x,y)=\rho\bigl(y-T(x)\bigr),
\qquad\text{so that}\qquad
\mathbb{P}(X_{n+1}\in dy\mid X_n=x)=k(x,y)\,dy.
\]
The corresponding transfer operator (acting on densities $f\in L^1(\mathbb{R}^d)$) is
\begin{equation}\label{eq:L-def-d}
(Lf)(y)
:=\int_{\mathbb{R}^d} f(x)\,k(x,y)\,dx
=\int_{\mathbb{R}^d} f(x)\,\rho\bigl(y-T(x)\bigr)\,dx.
\end{equation}
Fix $\alpha>0$ and consider the exponential Lyapunov weight
\begin{equation}\label{eq:V-def-d}
V(x):=\exp\big(\alpha\|x\|\big),\qquad x\in\mathbb{R}^d.
\end{equation}
Our goal in this section is to prove a Foster--Lyapunov drift condition
\begin{equation}\label{eq:FL-goal-d}
PV(x)\le \lambda V(x)+b
\qquad \text{for all } x\in\mathbb{R}^d,
\end{equation}
for some $\lambda\in(0,1)$ and $b<\infty$. For this we first estimate the Lyapunov function average  on a Gaussian distribution in Section \ref{gaussiancase} and we then look into the case of sub-Gamma distributions in Section \ref{extension}.

\subsection{Additive Gaussian noise on $\mathbb{R}^d$ and an exponential Lyapunov weight}
\label{gaussiancase}

\begin{lemma}\label{lem:gauss-shift-d}
Let $\xi\sim\mathcal{N}(0,\sigma^2 I_d)$ and let $\mu\in\mathbb{R}^d$ with $r:=\|\mu\|>0$.
Fix $\alpha>0$ and assume $\alpha\sigma^2<r$. Then
\begin{equation}\label{eq:gauss-shift-d-bound}
\mathbb{E}\exp\big(\alpha\|\mu+\xi\|\big)
\le
\exp(\alpha r)\,
\Bigl(1-\frac{\alpha\sigma^2}{r}\Bigr)^{-d/2}
\exp\!\Bigl(\frac{\alpha^2\sigma^2}{2\bigl(1-\frac{\alpha\sigma^2}{r}\bigr)}\Bigr).
\end{equation}
In particular, as $r\to\infty$ the second term in the right hand of \eqref{eq:gauss-shift-d-bound} converges to $\exp(\alpha^2\sigma^2/2)$.
\end{lemma}

\begin{proof}
Let $r=\|\mu\|>0$. Using the concavity inequality
$\sqrt{a+b}\le \sqrt{a}+\frac{b}{2\sqrt{a}}$ with $a=r^2$ and $b=2\langle \mu,\xi\rangle+\|\xi\|^2$ yields
\[
\|\mu+\xi\|
=\sqrt{r^2+2\langle \mu,\xi\rangle+\|\xi\|^2}
\le r + \Big\langle \frac{\mu}{r},\xi\Big\rangle + \frac{\|\xi\|^2}{2r}.
\]
Therefore,
\[
e^{\alpha\|\mu+\xi\|}
\le
e^{\alpha r}\,
\exp\!\Big(\alpha\Big\langle \frac{\mu}{r},\xi\Big\rangle+\frac{\alpha}{2r}\|\xi\|^2\Big).
\]
Let $u:=\mu/r$ so that $\|u\|=1$, and decompose $\xi=Zu+\eta$ where
$Z=\langle u,\xi\rangle\sim \mathcal{N}(0,\sigma^2)$ and $\eta$ is an independent Gaussian in $u^\perp\cong\mathbb{R}^{d-1}$
with law $\mathcal{N}(0,\sigma^2 I_{d-1})$. Then $\|\xi\|^2=Z^2+\|\eta\|^2$ and
\[
\alpha\Big\langle \frac{\mu}{r},\xi\Big\rangle+\frac{\alpha}{2r}\|\xi\|^2
=
\alpha Z + \frac{\alpha}{2r}Z^2 + \frac{\alpha}{2r}\|\eta\|^2.
\]
Hence the expectation factorizes:
\[
\mathbb{E}\exp\!\Big(\alpha Z+\frac{\alpha}{2r}Z^2\Big)\cdot
\mathbb{E}\exp\!\Big(\frac{\alpha}{2r}\|\eta\|^2\Big).
\]
These are standard Gaussian moment generating functions. Setting $t:=\frac{\alpha}{2r}$, the condition $\alpha\sigma^2<r$
is exactly $1-2t\sigma^2>0$. One computes (by completing the square in $Z$ and using the $\chi^2$ mgf for $\|\eta\|^2$):
\[
\mathbb{E}\exp\!\Big(\alpha Z+tZ^2\Big)
=
(1-2t\sigma^2)^{-1/2}\exp\!\Big(\frac{\alpha^2\sigma^2}{2(1-2t\sigma^2)}\Big),
\]
\[
\mathbb{E}\exp\!\big(t\|\eta\|^2\big)=(1-2t\sigma^2)^{-(d-1)/2}.
\]
Multiplying and recalling $1-2t\sigma^2=1-\frac{\alpha\sigma^2}{r}$ gives \eqref{eq:gauss-shift-d-bound}.
\end{proof}

Now we can prove the Foster-Lyapunov condition for $V(x)=e^{\alpha\|x\|}$ with a suitable $\alpha$.

\begin{proposition}[Exponential Foster--Lyapunov condition on $\mathbb{R}^d$]\label{prop:FL-exp-d}
Assume \eqref{eq:dissip-d} and \eqref{eq:local-bdd-d}. Then there exist
$\alpha>0$, $\lambda\in(0,1)$ and $b<\infty$ such that for $V(x)=e^{\alpha\|x\|}$,
\[
PV(x)\le \lambda V(x)+b\qquad \forall x\in\mathbb{R}^d.
\]
\end{proposition}

\begin{proof}
Fix $\alpha>0$ so small that
\begin{equation}\label{eq:alpha-choice-d}
q(\alpha):=\exp\!\Bigl(-\alpha C+\frac{\alpha^2\sigma^2}{2}\Bigr)<1.
\end{equation}
For instance, any $\alpha\in(0,2C/\sigma^2)$ satisfies \eqref{eq:alpha-choice-d}.
Choose $\lambda\in(q(\alpha),1)$.

\medskip
\noindent\textbf{Step 1 (estimate when $\|T(x)\|$ is large).}
Define for $r>\alpha\sigma^2$,
\[
F_\alpha(r):=
\Bigl(1-\frac{\alpha\sigma^2}{r}\Bigr)^{-d/2}
\exp\!\Bigl(\frac{\alpha^2\sigma^2}{2\bigl(1-\frac{\alpha\sigma^2}{r}\bigr)}\Bigr).
\]
By Lemma~\ref{lem:gauss-shift-d}, if $r=\|T(x)\|>\alpha\sigma^2$, then
\begin{equation}\label{eq:PV-lemma-d}
PV(x)=\mathbb{E}e^{\alpha\|T(x)+\xi\|}
\le e^{\alpha\|T(x)\|}\,F_\alpha(\|T(x)\|).
\end{equation}
Moreover $F_\alpha(r)\to e^{\alpha^2\sigma^2/2}$ as $r\to\infty$. Hence there exists $r_\ast>\alpha\sigma^2$ such that
\begin{equation}\label{eq:r-star-d}
e^{-\alpha C}\,\sup_{r\ge r_\ast}F_\alpha(r)\le \lambda.
\end{equation}
Now assume $\|x\|\ge R_0$ and $\|T(x)\|\ge r_\ast$. Using \eqref{eq:dissip-d}, \eqref{eq:PV-lemma-d}, and \eqref{eq:r-star-d},
\[
PV(x)\le e^{\alpha(\|x\|-C)}\sup_{r\ge r_\ast}F_\alpha(r)\le \lambda e^{\alpha\|x\|}=\lambda V(x).
\]

\medskip
\noindent\textbf{Step 2 (estimate when $\|T(x)\|$ is not large).}
If $\|T(x)\|\le r_\ast$, then $\|T(x)+\xi\|\le \|T(x)\|+\|\xi\|$ implies
\[
PV(x)=\mathbb{E}e^{\alpha\|T(x)+\xi\|}
\le e^{\alpha r_\ast}\,\mathbb{E}e^{\alpha\|\xi\|}
=: M_\alpha<\infty,
\]
where finiteness holds because $\xi$ is Gaussian and $\|\xi\|$ has finite exponential moments.
Since $V(x)=e^{\alpha\|x\|}\to\infty$ as $\|x\|\to\infty$, there exists $R_1\ge R_0$ such that
\begin{equation}\label{eq:R1-d}
\|x\|\ge R_1\quad\Longrightarrow\quad M_\alpha\le \lambda V(x).
\end{equation}
\medskip
\noindent\textbf{Step 3.} Combining Step~1 and Step~2, for all $\|x\|\ge R_1$ we have
\begin{equation}\label{eq:outside-d}
PV(x)\le \lambda V(x).
\end{equation}

\medskip
\noindent\textbf{Step 4 (Patching the drift on a compact set).}
Define
\[
b:=\sup_{\|x\|\le R_1} PV(x).
\]
This is finite by local boundedness \eqref{eq:local-bdd-d}: letting $M_{R_1}:=\sup_{\|x\|\le R_1}\|T(x)\|<\infty$,
\[
PV(x)=\mathbb{E}e^{\alpha\|T(x)+\xi\|}
\le e^{\alpha M_{R_1}}\,\mathbb{E}e^{\alpha\|\xi\|}<\infty,
\]
uniformly on $\{\|x\|\le R_1\}$. Finally, for $\|x\|\ge R_1$, \eqref{eq:outside-d} gives $PV(x)\le \lambda V(x)\le \lambda V(x)+b$,
while for $\|x\|\le R_1$ we have $PV(x)\le b\le \lambda V(x)+b$ since $V\ge 1$.
Thus \eqref{eq:FL-goal-d} holds for all $x\in\mathbb{R}^d$.
\end{proof}

\subsection{Extension to sub-exponential noise}
\label{extension}
A more general model for the noise is the sub-exponential model, which accounts for possibly heavier tails than the Gaussian noise. The noise in stochastic gradient methods is induced by the random selection of observations or mini-batches and therefore need not follow an exact Gaussian distribution. Moreover, individual stochastic gradients frequently involve products of random variables. Similar multiplicative structures arise in backpropagation through the products of activations, residuals, weights, and derivatives.
The sub-exponential assumption consequently allows for occasional large gradient deviations and accommodates a broader class of data distributions while retaining sufficiently strong concentration properties. such as Bernstein-type deviation inequalities.

We first establish the analogue of the Gaussian shift estimate that will be
used in the Foster--Lyapunov argument.

\begin{definition}
    A random vector in $\mathbb R^d$ is said to be sub-Gamma if there exists $b_\xi>0$ and $\nu>0$ such that for every $u\in\mathbb{R}^d$ with
$\|u\|\leq 1$, and every $s$ satisfying $|s|<1/b_\xi$,
\begin{equation}
\label{eq:uniform-subexp-mgf-d}
\mathbb{E}
\exp\bigl(s\langle u,\xi\rangle\bigr)
\leq
\exp(s^2\nu^2).
\end{equation}
\end{definition}

\begin{lemma}[Uniform large-shift estimate]
\label{lem:subexp-large-shift-d}
Let $\xi$ be a sub-Gamma random vector in
$\mathbb{R}^d$.
Then for $\alpha\in(0,1/(4b_\xi))$, 
\begin{equation}
\label{eq:large-shift-limit-d}
\limsup_{r\to\infty}\sup_{\|u\|=1}
\mathbb{E}
\exp\Bigl(
\alpha\bigl(\|ru+\xi\|-r\bigr)
\Bigr)
\leq
\sup_{\|u\|=1}
\mathbb{E}\exp\bigl(\alpha\langle u,\xi\rangle\bigr)
\leq
\exp(\alpha^2\nu^2).
\end{equation}
Moreover,
\begin{equation}
\label{eq:uniform-exp-moment-noise-d}
\mathbb{E}\exp\bigl(2\alpha\|\xi\|\bigr)
<\infty.
\end{equation}
\end{lemma}
\begin{proof}
We first prove \eqref{eq:uniform-exp-moment-noise-d}. Let
$\mathcal{N}_{1/2}$ be a $1/2$-net of the Euclidean unit sphere such that
\[
|\mathcal{N}_{1/2}|\leq 5^d.
\]
The standard net argument gives, for every $z\in\mathbb{R}^d$,
\[
\|z\|
\leq
2\max_{v\in\mathcal{N}_{1/2}}\langle v,z\rangle.
\]
Consequently,
\[
\exp\bigl(2\alpha\|z\|\bigr)
\leq
\sum_{v\in\mathcal{N}_{1/2}}
\exp\bigl(4\alpha\langle v,z\rangle\bigr).
\]
Since $4\alpha<1/b_\xi$, assumption
\eqref{eq:uniform-subexp-mgf-d} yields
\begin{align*}
\mathbb{E}\exp\bigl(2\alpha\|\xi\|\bigr)
&\leq
\sum_{v\in\mathcal{N}_{1/2}}
\mathbb{E}\exp\bigl(4\alpha\langle v,\xi\rangle\bigr)
\leq
5^d\exp(16\alpha^2\nu^2)
<\infty.
\end{align*}
We now prove the large-shift estimate \eqref{eq:large-shift-limit-d}. For every $r>0$, every unit vector
$u$, and every $z\in\mathbb{R}^d$, one has
\begin{equation}
\label{eq:norm-linearization-subexp-d}
0
\leq
\|ru+z\|-r-\langle u,z\rangle
\leq
\frac{\|z\|^2}{2r}.
\end{equation}
Indeed, the lower bound follows from
\[
\|ru+z\|
\geq
\langle u,ru+z\rangle
=
r+\langle u,z\rangle,
\]
whereas the upper bound follows from the concavity inequality for the
square root:
\begin{align*}
\|ru+z\|
&=
\sqrt{r^2+2r\langle u,z\rangle+\|z\|^2}
\leq
r+\langle u,z\rangle+\frac{\|z\|^2}{2r}.
\end{align*}
Fix $K>0$. On the event $\{\|\xi\|\leq K\}$,
\eqref{eq:norm-linearization-subexp-d} implies
\begin{align*}
0
&\leq
\exp\Bigl(
\alpha\bigl(\|ru+\xi\|-r\bigr)
\Bigr)
-
\exp\bigl(\alpha\langle u,\xi\rangle\bigr)
\leq
\exp(\alpha K)
\left[
\exp\left(\frac{\alpha K^2}{2r}\right)-1
\right].
\end{align*}
On the complementary event $\{\|\xi\|\leq K\}^c$, the triangle inequality gives
\[
\exp\Bigl(
\alpha\bigl(\|ru+\xi\|-r\bigr)
\Bigr)
\leq
\exp(\alpha\|\xi\|),
\]
and also
\[
\exp\bigl(\alpha\langle u,\xi\rangle\bigr)
\leq
\exp(\alpha\|\xi\|).
\]
It follows that
\begin{align*}
&\sup_{\|u\|=1}
\left|
\mathbb{E}
\exp\Bigl(
\alpha\bigl(\|ru+\xi\|-r\bigr)
\Bigr)
-
\mathbb{E}
\exp\bigl(\alpha\langle u,\xi\rangle\bigr)
\right|
\\
&\qquad\leq
\exp(\alpha K)
\left[
\exp\left(\frac{\alpha K^2}{2r}\right)-1
\right]
+
2
\mathbb{E}
\left[
\exp(\alpha\|\xi\|)
\mathbf{1}_{\{\|\xi\|>K\}}
\right].
\end{align*}
Using \eqref{eq:uniform-exp-moment-noise-d},
\begin{align*}
\mathbb{E}
\left[
\exp(\alpha\|\xi\|)
\mathbf{1}_{\{\|\xi\|>K\}}
\right]
&\leq
e^{-\alpha K}
\mathbb{E}\exp\bigl(2\alpha\|\xi\|\bigr).
\end{align*}
Therefore, by first letting $r\to\infty$ and then
$K\to\infty$, we obtain
\begin{align*}
\limsup_{r\to\infty}\mathbb{E}
\exp\Bigl(
\alpha\bigl(\|ru+\xi\|-r\bigr)
\Bigr)
&\leq
\sup_{\|u\|=1}
\mathbb{E}
\exp\bigl(\alpha\langle u,\xi\rangle\bigr)
\\
&\leq
\exp(\alpha^2\nu^2),
\end{align*}
where the last inequality follows from
\eqref{eq:uniform-subexp-mgf-d} with $s=\alpha$.
\end{proof}
We can now prove the exponential Foster--Lyapunov condition under the same
additive dissipativity assumption as in the Gaussian case.
\begin{proposition}
[Exponential Foster--Lyapunov condition for sub-exponential noise]
\label{prop:FL-exp-d-subgamma}
Assume that the following conditions hold.
\begin{enumerate}
    \item There exist $C>0$ and $R_0<\infty$ such that
    \begin{equation}
    \label{eq:dissip-subexp-d}
    \|T(x)\|
    \leq
    \|x\|-C,
    \qquad
    \|x\|\geq R_0.
    \end{equation}
    \item The map $T$ is locally bounded:
    \begin{equation}
    \label{eq:local-bdd-subexp-d}
    \sup_{\|x\|\leq R}\|T(x)\|<\infty,
    \qquad
    R<\infty.
    \end{equation}

    \item There exist $b_\xi>0$ and $\nu>0$ such that, for every
    $n\in\mathbb{N}^d$, every $u\in\mathbb{R}^d$ with $\|u\|\leq 1$,
    and every $s$ satisfying $|s|<1/b_\xi$,
    \begin{equation}
    \label{eq:uniform-subexp-d}
    \mathbb{E}
    \exp\bigl(s\langle u,\xi\rangle\bigr)
    \leq
    \exp(s^2\nu^2).
    \end{equation}
\end{enumerate}
Then there exist $\alpha>0$ with $0<\alpha<
\min\left\{
\frac{1}{4b_\xi},
\frac{C}{\nu^2}
\right\}$,
$\lambda\in(0,1)$ and $B<\infty$ such that,
for
\[
V(x):=\exp(\alpha\|x\|),
\]
one has
\begin{equation}
\label{eq:FL-subexp-goal-d}
PV(x)
\leq
\lambda V(x)+B,
\qquad
x\in\mathbb{R}^d.
\end{equation}
\end{proposition}
\begin{proof}
Fix $\alpha>0$ as prescribed. Then
\begin{equation}
\label{eq:q-alpha-subexp-d}
q(\alpha)
:=
\exp\bigl(-\alpha C+\alpha^2\nu^2\bigr)
<1.
\end{equation}
Choose $\lambda\in(q(\alpha),1)$.

\medskip
\noindent
\textbf{Step 1 (estimate when $\|T(x)\|$ is large).}
By
\eqref{eq:large-shift-limit-d},
\[
\limsup_{r\to\infty}
e^{-\alpha C}\mathbb{E}
\exp\Bigl(
\alpha\bigl(\|ru+\xi\|-r\bigr)
\Bigr)
\leq
\exp\bigl(-\alpha C+\alpha^2\nu^2\bigr)
=
q(\alpha)
<
\lambda.
\]
Consequently, there exists $r_\ast>0$ such that
\begin{equation}
\label{eq:r-star-subexp-d}
e^{-\alpha C}
\sup_{r\geq r_\ast}\mathbb{E}
\exp\Bigl(
\alpha\bigl(\|ru+\xi\|-r\bigr)
\Bigr)
\leq
\lambda.
\end{equation}
Suppose that $\|x\|\geq R_0$ and $\|T(x)\|\geq r_\ast$. Write
\[
T(x)=\|T(x)\|u_x,
\qquad
\|u_x\|=1.
\]
Then,
\begin{align*}
PV(x)
&=
\mathbb{E}
\exp\bigl(\alpha\|T(x)+\xi\|\bigr)
\\
&\leq
\exp\bigl(\alpha\|T(x)\|\bigr)
\ \mathbb{E}
\exp\Bigl(
\alpha\bigl(\|ru+\xi\|-r\bigr)
\Bigr).
\end{align*}
Using \eqref{eq:dissip-subexp-d} and
\eqref{eq:r-star-subexp-d}, we obtain
\begin{align*}
PV(x)
&\leq
\exp\bigl(\alpha(\|x\|-C)\bigr)
\sup_{r\geq r_\ast}\ \mathbb{E}
\exp\Bigl(
\alpha\bigl(\|ru+\xi\|-r\bigr)
\Bigr)
\\
&\leq
\lambda\exp(\alpha\|x\|)
=
\lambda V(x).
\end{align*}

\medskip
\noindent
\textbf{Step 2 (estimate when $\|T(x)\|$ is not large).}
Suppose that $\|T(x)\|\leq r_\ast$. By the triangle inequality,
\begin{align*}
PV(x)
&=
\mathbb{E}
\exp\bigl(\alpha\|T(x)+\xi\|\bigr)
\\
&\leq
\exp(\alpha r_\ast)
\mathbb{E}\exp(\alpha\|\xi\|).
\end{align*}
Lemma~\ref{lem:subexp-large-shift-d} gives
\[
M_\alpha
:=
\exp(\alpha r_\ast)
\ \mathbb{E}\exp(\alpha\|\xi\|)
<\infty.
\]
Hence
\[
PV(x)\leq M_\alpha
\qquad
\text{whenever }
\|T(x)\|\leq r_\ast.
\]
Since $V(x)=\exp(\alpha\|x\|)\to\infty$ as $\|x\|\to\infty$, there
exists $R_1\geq R_0$ such that
\begin{equation}
\label{eq:R1-subexp-d}
\|x\|\geq R_1
\quad\Longrightarrow\quad
M_\alpha\leq\lambda V(x).
\end{equation}
\medskip
\noindent
\textbf{Step 3.}
Let $\|x\|\geq R_1$. If $\|T(x)\|\geq r_\ast$, Step~1 gives
\[
PV(x)\leq\lambda V(x).
\]
If $\|T(x)\|\leq r_\ast$, Step~2 and
\eqref{eq:R1-subexp-d} give the same inequality. Thus
\begin{equation}
\label{eq:outside-subexp-d}
PV(x)
\leq
\lambda V(x),
\qquad
\|x\|\geq R_1.
\end{equation}
\medskip
\noindent
\textbf{Step 4.}
Set
\[
M_{R_1}
:=
\sup_{\|x\|\leq R_1}\|T(x)\|.
\]
By local boundedness, $M_{R_1}<\infty$. Therefore, for
$\|x\|\leq R_1$,
\begin{align*}
PV(x)
&\leq
\exp(\alpha M_{R_1})\ 
\mathbb{E}\exp(\alpha\|\xi\|)
\end{align*}
Define
\[
B
:=
\exp(\alpha M_{R_1})
\ \mathbb{E}\exp(\alpha\|\xi\|)
<\infty.
\]
Thus, the claim is proved after noticing that for $\|x\|\geq R_1$, \eqref{eq:outside-subexp-d} gives
\[
PV(x)\leq\lambda V(x)\leq\lambda V(x)+B.
\]
and for  $\|x\|\leq R_1$, the definition of $B$ gives
\[
PV(x)\leq B\leq\lambda V(x)+B.
\]
Thus
\[
PV(x)\leq\lambda V(x)+B,
\qquad
x\in\mathbb{R}^d,
\]
which proves \eqref{eq:FL-subexp-goal-d}.
\end{proof}

\section*{Acknowledgement}

S.G.  acknowledges the MIUR Excellence Department Project awarded to the
Department of Mathematics, University of Pisa, CUP I57G22000700001.
S.G. also thanks Université Lumière Lyon 2 for hospitality during the research.

%

\bibliographystyle{plain}
\bibliography{Biblio}

\end{document}